\documentclass[11pt]{article}

\usepackage{authblk}
\author{Ingvar Ziemann\thanks{Corresponding author: ingvarz@seas.upenn.edu}}
\author{Nikolai Matni}
\author{George J. Pappas}

\affil{University of Pennsylvania}
\usepackage[utf8]{inputenc} % allow utf-8 input
\usepackage[T1]{fontenc} %\usepackage{ae} \usepackage{aecompl}   % use 8-bit T1 fonts

\usepackage{fullpage}

\usepackage{url}            % simple URL typesetting
\usepackage{booktabs}       % professional-quality tables
\usepackage{amsfonts}       % blackboard math symbols
\usepackage{nicefrac}       % compact symbols for 1/2, etc.
\usepackage{microtype}      % microtypography
\usepackage[dvipsnames]{xcolor}

\usepackage{natbib} %use option numbers if arxiv fails
\usepackage{enumitem}
\usepackage{amsthm}
\usepackage{amssymb}
\usepackage{amsmath}
\usepackage{mathrsfs}
\usepackage{thmtools}
\usepackage{hyperref}
\hypersetup{
    pdftoolbar=true,        % show Acrobat’s toolbar?
    pdfmenubar=true,        % show Acrobat’s menu?
    pdffitwindow=false,     % window fit to page when opened
    pdfstartview={FitH},    % fits the width of the page to the window
    pdfnewwindow=true,      % links in new PDF window
    colorlinks=true,       % false: boxed links; true: colored links
    linkcolor=blue,          % color of internal links (change box color with linkbordercolor)
    citecolor=ForestGreen,        % color of links to bibliography
    filecolor=magenta,      % color of file links
    urlcolor=red,           % color of external links
    breaklinks=false,
}

\usepackage{cleveref}
\usepackage{thmtools}
\usepackage{thm-restate}

\usepackage{graphicx}
\usepackage{subcaption}

\numberwithin{equation}{section}

\date{}

\title{State space models, emergence, and ergodicity: \\How many parameters are needed for stable predictions?}

\newcommand{\e}{\varepsilon}

\newcommand{\E}{\mathbf{E}}

\DeclareMathOperator{\blkdiag}{blkdiag}

\DeclareMathOperator{\tr}{tr}
\DeclareMathOperator{\VEC}{\mathsf{vec}}

\newcommand{\scrF}{\mathscr{F}}

\newcommand{\scrP}{\mathscr{P}}

\newcommand{\sfP}{\mathsf{P}}
\newcommand{\sfQ}{\mathsf{Q}}
\newcommand{\sfM}{\mathsf{M}}

\newcommand{\dx}{d_{\mathsf{X}}}
\newcommand{\dm}{d_{\mathsf{M}}}
\newcommand{\dy}{d_{\mathsf{Y}}}

\newcommand{\R}{\mathbb{R}}
\newcommand{\N}{\mathbb{N}}

\newcommand{\iid}{iid}

\newtheorem{theorem}{Theorem}[section] % same for example numbers
\newtheorem*{theorem*}{Theorem} % same for example numbers
\newtheorem{proposition}{Proposition}[section] % same for example numbers
\newtheorem{assumption}{Assumption}[section] % same for example numbers
\newtheorem{remark}{Remark}[section]
\newtheorem{lemma}{Lemma}[section]

\begin{document}

\maketitle

\begin{abstract} 
How many parameters are required for a model to execute a given task? It has been argued that large language models, pre-trained via self-supervised learning, exhibit emergent capabilities such as multi-step reasoning as their number of parameters reach a critical scale. In the present work, we explore whether this phenomenon can  analogously be replicated in a simple theoretical model. We show that the problem of learning linear dynamical systems---a simple instance of self-supervised learning---exhibits a corresponding phase transition. Namely, for every non-ergodic linear system there exists a critical threshold such that a learner using fewer parameters than said threshold cannot achieve bounded error for large sequence lengths. Put differently, in our model we find that tasks exhibiting substantial long-range correlation require a certain critical number of parameters---a phenomenon akin to emergence. We also investigate the role of the learner's parametrization and consider a simple version of a linear dynamical system  with hidden state---an imperfectly observed random walk in $\R$. For this situation, we show that there exists no learner using a linear filter which can succesfully learn the random walk unless the filter length exceeds a certain threshold depending on the effective memory length and horizon of the problem. 

%\Ingvar{Nik's critique: Does this completely disappear if we allow nonlinear thingies. Ingvar's answer: I'm not sure how much I agree with this since part of the story is certainly that PARAMETRIZATION MATTERS. However, I do agree that we should have some nonlinear experiments to hedge}
\end{abstract}

\section{Introduction}

Consider a pre-trained large language model (LLM) obtained via self-supervised learning by predicting the next word or token. While the performance on pre-training loss exhibits rather predictable behavior \citep{kaplan2020scaling}, \citet{wei2022emergent} observe that such models often exhibit a phase transition in their downstream capabilities as the number of trainable parameters (or training FLOPs) reaches a critical scale---they exhibit emergent capabilities such as successful in-context learning \citep{brown2020language}. While these models are typically extremely large in terms of their number of parameters, a recent line of work has shown that such behavior can also be recovered in smaller models by considering appropriately simplified tasks \citep{allen2024physics}. Here, we offer a possible mechanistic explanation for this phenomenon by restricting to a simple class of auto-regressive learning models.

Namely, we point out that certain tasks---or more precisely, predicting in certain generative models---exhibiting long-range correlations and a lack of ergodicity can only be executed successfully once model scale reaches a certain critical threshold. One may think of our result  as the bias term in the bias-variance trade-off exhibiting a sharp jump---a phase transition---depending on whether the model class is rich enough to be fully descriptive of this lack of stochastic stability. We illustrate this phenomenon by a simple problem: learning a  linear dynamical system. Incidentally, such linear systems are also fundamental building blocks in the increasingly popular state state model architectures for sequence modelling---an alternative to the popular transformer architecture \citep{vaswani2017attention, gu2022efficiently}.

%\Ingvar{Lets instead introduce the general partially observed model here and then consider the below results as specializations.

%Moreover, let us add "a layer of complexity" namely, lets suppose that the learner actually observes $Z = F(Y)$. By the data-processing (in)equality, if $F$ is a bijection this actually does not change the loss function at all (if we use KL divergence)---may also be worth reasoning more generally about the role of tokenization here!}

Here and in the sequel we study generative modelling of tasks $\sfP_Z$ corresponding to distributions over sequences of tokens $Z_{1:T}$. A learner has pre-trained a (compressed) generative model  $\sfQ_Z$ using data not necessarily coming from $\sfP_Z$. The performance of such a model $\sfQ_Z$ on a task $\sfP_Z$ will be measured by its divergence from the ground truth:
\begin{equation}\label{eq:kllossdefined}
  \mathrm{d}_{\mathsf{KL}}( \sfP_Z \| \sfQ) = \E_{\sfP} \log \frac{\mathrm{d}\sfP_Z}{\mathrm{d}\sfQ}.
\end{equation}
We ask the following question:

\begin{quote}
    \textbf{Q: \:}\textit{Suppose that $\sfQ$ comes from a parametric hypothesis class. Does there exist a critical threshold in terms of the number parameters such that $T^{-1}\mathrm{d}_{\mathsf{KL}}( \sfP_Z \| \sfQ) \to \infty $ as $T \to \infty$ unless the parameter count exceeds said threshold?}
\end{quote}

In other words, we ask whether a given task-hypothesis class combination admits \emph{stable learners}---learners for which the KL-risk does not diverge as the sequence length $T$ becomes long (notice that the normalization $T^{-1}$ is necessary to avoid trivial behavior for product measures). Our view here is that language, arriving in discrete packages such as articles and books, is non-ergodic when viewed at the package level. In this view, a single book forms a single trajectory of data in which the first word (or token) is the first data point and the last word the last data point. The distribution of words in the beginning of the book (introducing the suspects)  may well be quite different from the distribution at the end of the book (who did it?)---there is different meaning to be conveyed.

It is our hypothesis that it is exactly this lack of ergodicity that leads to emergent behavior. Our main simplifying assumption in relating non-ergodicity to model complexity is that the task $\sfP_Z$ has a latent state space model representation.

\begin{assumption}\label{ass:ssmodel}
    The $Z_{1:T}$ is in bijection to a state space model. More precisely, there exists a bijection $g$ such that $Z_t = g(Y_t)$ for $t \in [T]$ where $Y_{1:T}$ is generated by:
    \begin{equation}\label{eq:polds}
    \begin{aligned}
        X_{t+1}&=A_\star X_t +W_{t+1},\quad X_1=W_1 & \qquad
        Y_t & = C_\star X_t +V_{t}.
    \end{aligned}
\end{equation}
where $A_\star \in \mathbb{R}^{\dx\times \dx}, C_\star \in \R^{\dy \times \dx}$. Here, $W_{1:T+1}$ and $V_{1:T}$ are jointly Gaussian, mutually independent with block-diagonal covariance ($\Sigma_{W_1},\Sigma_W \otimes I_{T}, \Sigma_V \otimes I_{T}$) and mean zero.
\end{assumption}

Models of this form are standard in time series prediction tasks and systems modelling, but have also recently been popularized as building blocks in LLMs \citep{gu2022efficiently}.

Under \Cref{ass:ssmodel}, a version of the maximum entropy principle yields the following. For every nondegenerate distribution $\sfQ_Z$ over $Z_{1:T}$ under \Cref{ass:ssmodel} the following are true:
\begin{itemize}
    \item For $Y_{1:T}\sim\sfP_Y=\sfP_{g^{-1}(Z)}$ then: \begin{equation}
        \mathrm{d}_{\mathsf{KL}}( \sfP_Z \| \sfQ_Z) =\mathrm{d}_{\mathsf{KL}}( \sfP_Y \| \sfQ_{g^{-1}(Z)}).
    \end{equation}
    \item The Gaussian measure $\sfQ_Y$ with the same mean and covariance as $\sfQ_{g^{-1}(Z)}$ satisfies
    \begin{equation}
       \mathrm{d}_{\mathsf{KL}}( \sfP_Y \| \sfQ_{g^{-1}(Z)})\geq \mathrm{d}_{\mathsf{KL}}( \sfP_Y \| \sfQ_Y).
    \end{equation}
\end{itemize}
The first statement follows by bijection and the second statement is simply observing that Gaussian measures minimize KL subject to constraints on the first two moments. Our next observation is the standard (trivial yet powerful!) equivalence between generative modeling and next-token-prediction. Namely the generative modelling error on the left hand side below can be expanded in terms of the KL divergence chain rule:
\begin{equation}\label{eq:gaussianklexpanded}
\begin{aligned}
    \mathrm{d}_{\mathsf{KL}}( \sfP_Y \| \sfQ_Y)&= \sum_{t=1}^{T} \E_{\sfP} \log \frac{\mathrm{d}\sfP_Y^{t|1:t-1}}{\mathrm{d}\sfQ_Y^{t|1:t-1}}
    \\
    &=
    \frac{1}{2}\sum_{t=1}^T\left[ \| \E^{t-1}_{\sfP} Y_t-\E^{t-1}_{\sfQ} Y_t\|^2_{\Sigma^{-1/2}_{\sfQ_t}}+\tr\left(\Sigma_{\sfQ_t}^{-1} \Sigma_{\sfP_t}\right)-\log \det \left(\Sigma_{\sfQ_t}^{-1} \Sigma_{\sfP_t}\right)-\dy\right].
\end{aligned}
\end{equation}
It is reasonable to assume that $m I \preceq \Sigma_{\sfQ_t} \preceq M I $ for some universal constants $m,M$. Otherwise, either the term  $\tr\left(\Sigma_{\sfQ_t}^{-1} \Sigma_{\sfP_t}\right)$ grows unbounded (as we will see that $\Sigma_{\sfP_t}$ is well-conditioned in our examples), or the variance of the predictor becomes arbitrarily large. Combining the above we have that
\begin{equation}\label{eq:gaussiankllowerbounded}
    \mathrm{d}_{\mathsf{KL}}(\sfP_Z \| \sfQ) \gtrsim \sum_{t=1}^T\E_{\sfP }\| \E^{t-1}_{\sfP} Y_t-\E^{t-1}_{\sfQ} Y_t\|^2.
\end{equation}

It will be convenient to denote 
\begin{equation}\label{eq:elltdefined}
    \ell_T(\scrF, \scrP) \triangleq \inf_{\sfQ \in \scrF} \sup_{\sfP\in\scrP}\E_{\sfP } \sum_{t=1}^{T-1}\| \E^{t-1}_{\sfP} Y_t-\E^{t-1}_{\sfQ} Y_t\|^2.
\end{equation}
By the above reasoning via \eqref{eq:gaussianklexpanded}-\eqref{eq:gaussiankllowerbounded}, $\ell_T$ defined above in \eqref{eq:elltdefined} constitutes a lower bound on the KL-divergence risk \eqref{eq:kllossdefined} in which a learner---by picking a hypothesis in $\scrF$---competes with an adversary selecting a generative model from $\scrP$. Thus, imposing these additional constraints above, an instantiation of the above question \textbf{Q} becomes as follows. 

\begin{quote}
    \textbf{Q': \:}\textit{Fix a  family of  parametric hypothesis classes $\{\scrF_d\}_{d\in\N}$ and a family of possible generative models $\scrP$. Does there exist a critical threshold $d_\star$ in terms of the number parameters such that $$ T^{-1}\ell_T(\scrF_d, \scrP) \to \infty $$
    as $T \to \infty$ unless the parameter count exceeds said threshold ($d>d_\star$)?}
\end{quote} 
In the sequel we focus on identifying task-hypothesis pairs ($\scrP$, $\{\scrF_d\}_{d\in\N}$) where this divergence occurs. We will think of a task as exhibiting emergent behavior if it admits a nontrivial threshhold $d_\star$ mentioned in \textbf{Q'} above. 

Finally, before we proceed let us also remark that there is some degree of necessity to our choice of considering an adversarial model class $\scrP$ that we use to obtain meaningful lower bounds. To make this concrete, consider a parametric class of distributions $\scrP$ parametrized by some set of parameters, say $\theta\in \scrP$. Suppose the generative model corresponds to the parameter $\theta_\star$.  As long as $\scrF$ contains this parameter the only lower bound that can be obtained without including the supremum in \eqref{eq:elltdefined} is $0$. In other words, we need to model the fact that the learner does not have access to the parameter a priori. We accomplish this by letting an adversary pick a parameter against which the learner must compete.

%is an element of the $0$-dimensional space $\{\theta_\star\}$ and so learning this parameter is only hard  if   this $0$-dimensional space is not known to the learner a priori---we precisely model this situation by letting the adversary pick a parameter.% In the sequel we will also explore a different source of hardness:  the class $\scrF$ not being fully descriptive of $\scrP$.

%We begin by noting that every parameter, say $\theta_\star\in \R^{d_{\max}}$, lies in some $d$-dimensional submanifold, for any $d<d_{\max}$ since it certainly lies in the $0$-dimensional submanifold $\{\theta_\star\}$. In other words: for every $d$ there exists some $\scrF_d$ of the same dimension rendering a lower bound for a singleton class $\scrP$ trivial (equal to $0$ at best). 

%Hence, on the one hand, if we would like to consider the situation in which $\scrF_d$ ranges over \emph{every} $d$-dimensional submanifold of $\scrP$ the problem becomes vacuous unless we consider an adversary allowing at least infinitesimal perturbations to the generative model. On the other hand, if we do not allow for this situation, our claims will depend strongly on the parametrization of the model class. 

\section{Contribution}

Our contributions can be stated informally as follows. 

\begin{theorem*}[Informal version of \Cref{thm:sfemergence}]
    There is emergent behavior in learning non-ergodic auto-regressive models: in a simple linear dynamical system with fully observed state, there exists no successful learner without using at least as many parameters as the squared number of (marginally) unstable eigenvalues.\footnote{Note that \Cref{thm:sfemergence} gives a more nuanced statement in terms of Jordan blocks---the above statement corresponds to the worst case Jordan block structure.} By contrast, this is possible once the parameter count exceeds said threshhold.
\end{theorem*}

\begin{figure}[htbp]
    \centering
    \begin{subfigure}{0.45\textwidth}
        \centering
        \includegraphics[width=\linewidth]{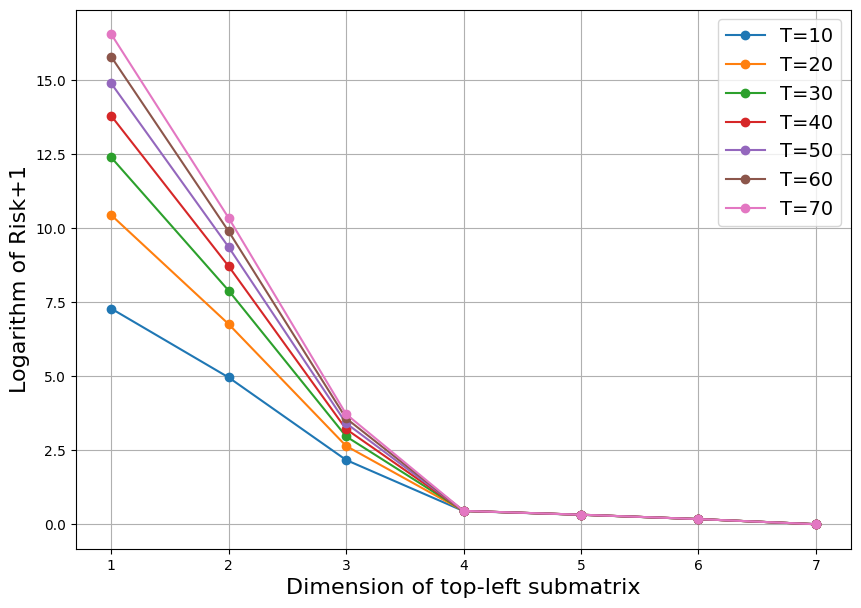}
        \caption{$A_\star$ chosen as the block-diagonal matrix consisting of first a Jordan block of size 4 with eigenvalue $1$ and second the rescaled identity of size 3 with eigenvalue $0.4$.}
        \label{fig:first_plot}
    \end{subfigure}
    \hfill
    \begin{subfigure}{0.45\textwidth}
        \centering
        \includegraphics[width=\linewidth]{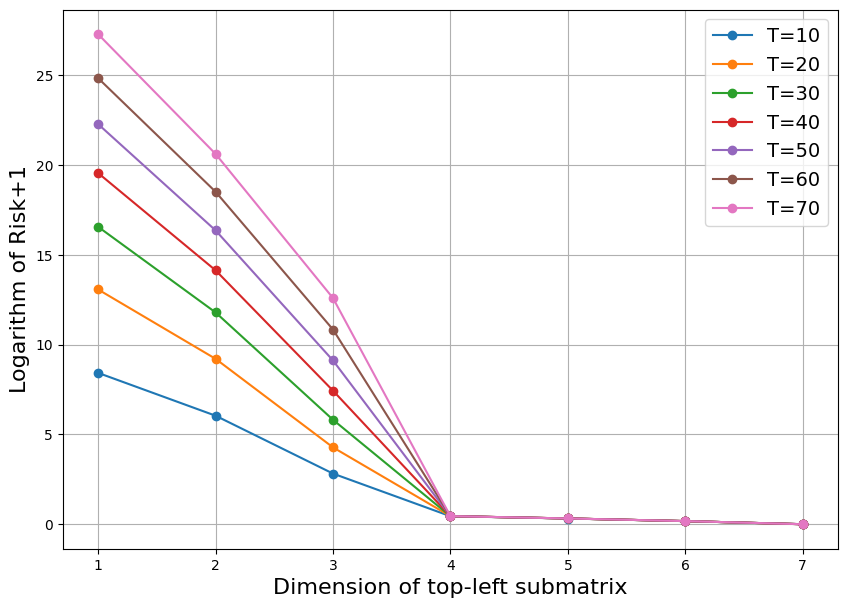}
         \caption{$A_\star$ chosen as the block-diagonal matrix consisting of first a Jordan block of size 4 with eigenvalue $1.1$ and second the rescaled identity of size 3 with eigenvalue $0.4$.}
        \label{fig:second_plot}
    \end{subfigure}
    \caption{We illustrate \Cref{thm:sfemergence} by a simple numerical example of learning a $\dx$-dimensional linear system $X_{t+1}=A_\star X_t+W_t$ with fewer than the required number of parameters and where $\dx=7$. Namely, we only estimating the top-left $k\times k$ sub-matrix, $k\in [\dx]$. We run the least squares estimator for samples drawn from $m \gg 2^{\dx}$ many trajectories and vary the trajectory length, $T$. As predicted, as $T$ grows the risk diverges unless the parametrization is sufficiently high-dimensional, $k=4$, at which the point the risk drops to near zero and exhibits more stable behavior (note the logarithmic scale on the $y$-axis).}
    \label{fig:side_by_side_plots}
\end{figure}

We also prove an extension to \Cref{thm:sfemergence} that applies to imperfect state observations but is restricted to learning in parametric classes consisting of finite-dimensional filters.

\begin{theorem*}[Informal version of \Cref{thm:stablepoldsthm}]
    For an imperfectly observed random walk in $\R$, there exists no successful learner in the class of linear filters unless the filter length exceeds a certain threshhold based on the detectability and horizon of the problem.
\end{theorem*}

\begin{remark}
    As a byproduct of our analysis, we note in passing that \Cref{thm:stablepoldsthm} shows that the truncation level used in \cite{tsiamis2019finite} for improper linear system identification cannot be much improved in general. In particular, improper learning with a finite length filter always (unless further constraints are added to the hypothesis class) incurs an extra approximation-theoretically induced logarithmic factor as opposed to the maximum likelihood estimator.
\end{remark}

\section{Emergence in Fully Observed Systems}

As a first example, let us consider a fully observed state space model. In this case, $C_\star$ in \eqref{ass:ssmodel} is simply the identity and $V_t$ is identically zero: 
\begin{equation}\label{eq:genmodlds}
    X_{t+1}= A_\star X_t +W_t, \qquad t=1,\dots,T-1, \qquad X_1=W_0
\end{equation}
We consider the setting in which a learner observes the trajectory $X_{1:T}$ and seeks to learn the generative model by recovering $A_\star$. We suppose that each $\scrF_d$ is given by a map $A_d: \sfM \mapsto \mathbb{R}^{\dx \times \dx}$ such that $\E^{t-1}_\sfQ Y_t = A(\theta) X_t$ where $\sfM$ is some smooth manifold of dimension $d_{\sfM}$. In this case the prediction risk becomes
\begin{equation}
    \sum_{t=1}^T\E_{\sfP }\| \E^{t-1}_{\sfP} Y_t-\E^{t-1}_{\sfQ} Y_t\|^2 = \sum_{t=1}^{T-1} \E \|A_d(\theta) X_t-A_\star X_t\|^2.
\end{equation}

%Namely, the learning objective is to minimize the loss or risk (the normalized negative log-likelihood):
%\begin{equation}\label{eq:theloss}
%    \ell(A:A_\star) \triangleq \frac{1}{T-1}\sum_{t=1}^{T-1} \E \|X_{t+1}-AX_t\|^2
%\end{equation}
%over some compact subset  $\mathsf{M} \subset \mathbb{R}^{\dx\times \dx}$. We will want to think of the dimension $\dx$ as large, and suppose the number of trainable parameters to be much smaller than $\dx\times \dx$; $\dm\triangleq \dim(M) \ll \dx\times \dx$.

%In this scenario, the loss \eqref{eq:theloss} can be decomposed as: \Ingvar{forgot $-\sigma^2$ on LHS}
%\begin{equation}\label{eq:lossdecomp}
%    \ell(A:A_\star) = \underbrace{\left[\ell(A:A_\star)-\ell(A_\sfM:A_\star)\right]}_{\textnormal{Excess Risk}}+\underbrace{\left[\ell(A_\sfM:A_\star)-\ell(A_\star:A_\star)\right]}_{\textnormal{Irreducible Risk}} \quad \textnormal{where}\quad A_{\sfM}\in \argmin_{A\in\sfM} \ell(A:A_\star).
%\end{equation}
%We henceforth denote the irreducible risk of the manifold $\sfM$ by $\ell_M(A_\star) \triangleq \ell(A_\sfM:A_\star)-\ell(A_\star:A_\star)$ with $A_\sfM$ as in \eqref{eq:lossdecomp} above. Before we proceed, let us make precise the notion that the generative model \eqref{eq:genmodlds} is not ergodic. 

\begin{assumption} \label{ass:nonergodicity}
    The spectral radius of $A_\star$ is at least unity.
\end{assumption}

We now show that when the generative model \eqref{eq:genmodlds} is not ergodic---\Cref{ass:nonergodicity} holds---the risk exhibits a phase transition in how it scales with the trajectory length $T$ as a function of the number of trainable parameters---the dimension of $\sfM$, $\dm$. In both cases below we abuse notation and write $\ell_T(\sfM, A_\star) = \ell_T(\sfM, \sfP_\star)$ where $\sfP_\star$ is the distribution of  $X_{1:T}$ with the parametrizing matrix $A_\star$ in the generative model \eqref{eq:genmodlds}.

\begin{theorem}\label{thm:sfemergence}
    Impose \Cref{ass:nonergodicity}. Let $d_\star^2$ be the sum of the squares of the algebraic multiplicities of all eigenvalues of $A_\star$ with magnitude at least unity. 

    \begin{enumerate}
        \item If $\dm < d_\star^2$, then for every $\e>0$ there exists $A_\star^\e \in \R^{\dx \times \dx}$ with $\|A_\star^\e-A_\star\|\leq \e$ such that:
        \begin{equation}
        \lim_{T\to \infty } T^{-1} \ell_T(\sfM, A_\star^\e) = \infty.
    \end{equation}
    \item If $\dm  < d_\star^2-d_{\star,1}$, there exists invertible $P$ such that:
    \begin{equation}
        \lim_{T\to \infty } T^{-1}\ell_T(\sfM,P^{-1}A_\star P) = \infty
    \end{equation}
    where $d_{\star,1}$ is the sum of the algebraic multiplicities of the  of all eigenvalues of $A_\star$ with magnitude at least unity.
    \end{enumerate}
\end{theorem}

The first part of \Cref{thm:sfemergence} shows that unless the number of parameters is quadratic in the number of unstable modes, there exists no learner with bounded loss that is robust to infinitesimal perturbations of the generative model. It is also interesting to note that learning becomes drastically more difficult if $A_\star$ has a single large Jordan block as opposed to being diagonal in some basis. The second part shows that this remains true even if the spectrum is fixed a priori and the perturbations are restricted to a change of basis. By contrast, if $A_\star$ is strictly stable, it is easy to see that there always exists $\e>0$ such that if $\|A_\star^\e-A_\star\|\leq \e$ it holds that  $\lim_{T\to \infty } \ell_\sfM(A_\star^\e) < \infty$---and this holding is independent of $\dm$.

\begin{proof}
First, we observe that for some invertible matrix $P_\e$ we may write $A_\star^\e = P^{-1}_\e J_{\star}^\e P_\e$ for the Jordan normal form of $A_\star^\e$ and where $J_{\star}^\e$ is block diagonal with the eigenvalues of $A_\star^\e$ on its main diagonal. The model \eqref{eq:genmodlds} can thus equivalently be written as
\begin{equation}
    \label{eq:genmodblockdiag}
    \underbrace{P_\e X_{t+1}}_{\triangleq H_{t+1}}=J_{\star}^\e\underbrace{P_\e X_{t}}_{\triangleq H_t}+\underbrace{P_\e W_t}_{\triangleq V_t}
\end{equation}
and \eqref{eq:genmodblockdiag} can unrolled as
\begin{equation}\label{eq:pspacecov}
    H_{t}=\sum_{k=0}^{t-1} (J_{\star}^\e)^k V_{t-k-1} \quad \textnormal{with}\quad \E H_tH_t^\dagger = \sum_{k=0}^{t-1}(J_{\star}^{\e})^k P_\e P^\dagger_\e  (J_{\star}^{\e})^{k,\dagger}
\end{equation}
where $\dagger$ denotes conjugate transpose.

Second, we observe that we may restrict attention without loss of generality to the situation in which $A_\star$ has a single repeated eigenvalue with multiplicity $d_\star$ by decomposing the system \eqref{eq:genmodlds} into its distinct $A_\star$-invariant subspaces. The general lower bound then follows by summing each of the individual subspace lower bounds.

Third, we notice that
\begin{equation}
    \begin{aligned}
        &\min_{A\in \sfM}\frac{1}{T-1} \sum_{t=1}^{T-1} \E \| (A-A_\star^\e)  X_t\|^2\\
        &=\min_{A\in \sfM}\frac{1}{T-1} \sum_{t=1}^{T-1} \E \| ( A -P^{-1}_eJ_{\star}^\e P_\e)  X_t\|^2
        &&(A_\star^\e =P^{-1}_eJ_{\star}^\e P_\e )
        \\
        &=\min_{A\in \sfM}\frac{1}{T-1} \sum_{t=1}^{T-1} \E \| (P^{-1}_e P_\e  AP^{-1}_e P_\e -P^{-1}_eJ_{\star}^\e P_\e)  X_t\|^2 
        && ( P^{-1}_e  P_\e = I )\\
        &
        = \min_{J\in P_\e \sfM P_\e^{-1}}\frac{1}{T-1} \sum_{t=1}^{T-1} \E \|P^{-1}_\e (J-J_{\star}^\e)PX_t\|^2 
        && (J \triangleq P_\e  AP^{-1}_e  )\\
         &
        = \min_{J\in P_\e \sfM P_\e^{-1}}\frac{1}{T-1} \sum_{t=1}^{T-1} \E \|P^{-1}_\e (J-J_{\star}^\e)H_t\|^2 
        && (P_\e X_t = H_t  )\\
        &
        \geq 
         \lambda_{\min}^2(P^{-1}) \min_{J\in P_\e \sfM P_\e^{-1}}\frac{1}{T-1} \sum_{t=1}^{T-1} \E \|(J-J_{\star}^\e)H_t\|^2
         \\
         &
         \geq
\lambda_{\min}^2(P^{-1})\lambda_{\min}^2(P)   \min_{J\in P_\e \sfM P_\e^{-1}}\frac{1}{T-1} \sum_{t=1}^{T-1} \E \tr \left(  \sum_{k=0}^{t-1} (J_{\star}^\e)^k (J_{\star}^\e)^{\dagger,k}  (J-J_{\star})(J-J_{\star})^\dagger\right). &&\eqref{eq:pspacecov}
    \end{aligned}
\end{equation}
Now the $d_\star$-many of the diagonal elements of each $ (J_{\star}^\e)^k (J_{\star}^\e)^{\dagger,k} $ are at least unity. Consequently all the $d_\star$-many of the diagonal elements of $  \sum_{k=0}^{t-1}  (J_{\star}^\e)^k (J_{\star}^\e)^{\dagger,k}  $ are larger than or equal to $t$. Indeed for some $|\lambda|\geq 1$ we have $J_\star^\e = (\lambda I_{d_\star}+N)$ for some nilpotent matrix $N$ and consequently $ \lim_{t\to \infty} \frac{1}{t} \lambda_{\min} \left( \sum_{k=0}^{t-1}  (J_{\star}^\e)^k (J_{\star}^\e)^{\dagger,k} \right)> 0$. Since both this matrix and $ (J-J_{\star}^\e)(J-J_{\star}^\e)^\dagger$ are positive semi-definite, it follows that $\ell_{\sfM
}<\infty$ if and only if there exists $ J \in P_\e \sfM P_\e^{-1}$ with $J =J_\star^\e$. 

To finish the proof notice that:
\begin{equation}\label{eq:solvabilityLDS}
    \begin{aligned}
        &\exists J \in P_\e \sfM P_\e^{-1}: &J =J_\star^\e,\\
        \Leftrightarrow \:
        &\exists A \in \sfM:  & P_\e AP^{-1}_\e = J_\star^\e,\\
        \Leftrightarrow\:
        &\exists A \in \sfM:  & A = P_\e^{-1}J_\star^\e P_\e=A_\star^\e.
    \end{aligned}
\end{equation}
The  first part of the result follows since $A_\star^\e$ varies over a $d_\star^2$-dimensional manifold and $A$ varies over a $\dm$-dimensional manifold. Hence, for \eqref{eq:solvabilityLDS} to have a solution for every $A_\star^\e$ we require that $\dm \geq d_\star^2$.

Now for the second part, let instead $P$ vary over the general linear group. In this case, $P^{-1}J_\star P$ varies over a $(d_\star^2-d_\star)$-dimensional manifold whereas $A \in \sfM$ only varies over a $\dm$-dimensional manifold. To see that the degrees of freedom of $P^{-1}J_\star P$ are indeed  $d_\star^2-d_\star$, invoke the Orbit-Stabilizer Theorem and notice that the dimension of the orbit of $J_\star$ under conjugation by the general linear group is equal to the dimension of the quotient space of the general linear group modulo the centralizer of $J_\star$. The general linear group has dimension $d_\star^2$ and the centralizer of a Jordan block under this action has dimension $d_\star$. Consequently, this equation cannot have a solution for every admissible choice of right hand side unless $\dm\geq d_\star^2-d_\star$. 
\end{proof}

\section{Hidden States and the Role of the Parametrization}

In \Cref{thm:sfemergence} we saw that we require a quadratic amount of parameters in the number of unstable modes. However, this was assuming direct access to the internal system state. If instead the state is hidden, the observations are no longer Markovian and exhibit longer range memory. We will now turn to investigating the appearance of such memory interacts with the potential instability (non-ergodicity) of $A_\star$. Let us also restrict attention to hypothesis classes consisting of finite-dimensional filters of the form $f_t(Y_{1:t-1})=\sum_{k=1}^{h} F_k Y_{t-k}$ for every $t$ (where $F_k$ is the decision-variable that does not depend on $t$). Finite memory of this type is present in many popular architectures, including transformers, where it is referred to as the context length \citep{vaswani2017attention}. We denote these classes $\sfM_{h}$. In this setting, for a fixed integer $h$ and hypothesis $f\in \sfM_{h}$, with representation $F_{1:h}$, we have that:
\begin{equation}\label{eq:poldsrisk}
     \sum_{t=1}^T\E_{\sfP }\| \E^{t-1}_{\sfP} Y_t-\E^{t-1}_{\sfQ} Y_t\|^2 = \sum_{t=1}^{T-1} \E \left\|\sum_{k=1}^{h} F_k Y_{t-k}-\E[Y_t | Y_{1:t-1}] \right\|^2.
\end{equation}
At this stage it must be pointed out that it is not just the dimensionality of the parametrization that matters but also the parametrization itself. There certainly exists a hypothesis class using no more than $\dx(\dx +\dy)$-many parameters rendering \eqref{eq:poldsrisk} null. On the other hand, the dimension of the internal state may be large or not even known a priori in which case it is appropriate to approximate \eqref{eq:polds} by a finite-dimensional filter---the question then becomes: \emph{what is the minimal filter length such that \eqref{eq:poldsrisk} remains stable?}

The analysis in the sequel passes via the Kalman filter. The next assumption guarantees that this can be represented by a linear time-invariant system. The part of the assumption dealing with time-invariance does not meaningfully restrict the generality of our results since the filter parameters convergence to their steady-state values at a super-exponential rate. 

\begin{assumption}\label{ass:ss}
The pair ($C,A$) is observable and $\Sigma_W,\Sigma_V \succ 0$. Moreover, the covariance of the initial state satisfies $\Sigma_{W_1}=\Sigma_{\mathrm{ss}}$, where $\Sigma_{\mathrm{ss}}$ solves the filter discrete algebraic Riccati equation.
\end{assumption}
Under \Cref{ass:ss} we have that 
    \begin{equation}
        \E[Y_t | Y_{1:t-1}] = \sum_{k=1}^{t-1} M^\star_k Y_{t-k} = M_\star \mathbf{Z}^{T-t}Y_{1:T-1}
    \end{equation}
    where $M^\star_k = C_\star (A_\star-L_\star C_\star)^kL_\star$ for some matrix $L_\star$ known as the \emph{Kalman gain} and accordingly $M_\star = C_\star\begin{bmatrix}
         (A_\star-L_\star C_\star) & \cdots &  (A_\star-L_\star C_\star)^{T-1}
    \end{bmatrix} $ and $\mathbf{Z}$ is the downshift operator. Similarly 
    \begin{equation}
       \sum_{k=1}^{h} F_k Y_{t-k} = F \begin{bmatrix}
           0_{t-h-1} & I_h & 0_{T-t-1} 
       \end{bmatrix} Y_{1:T-1} 
       =
       F \begin{bmatrix}
           0_{T-h-1} & I_h 
       \end{bmatrix}
       \mathbf{Z}^{T-t} Y_{1:T-1}
       =FE_h\mathbf{Z}^{T-t} Y_{1:T-1}
    \end{equation}
where $F=\begin{bmatrix}
    F_h & \cdots & F_1
\end{bmatrix}$ and $E_h = \begin{bmatrix}
           0_{T-h-1} & I_h 
       \end{bmatrix} $. This conveniently allows us to lower-bound the prediction risk via the following closed form. 
\begin{equation}\label{eq:poldsrelax}
    \begin{aligned}
     \min_{F_{1:h}}\sum_{t=1}^{T-1} \E \left\|\sum_{k=1}^{h} F_k Y_{t-k}-\E[Y_t | Y_{1:t-1}] \right\|^2
     &
     \geq \sum_{t=1}^{T-1}\min_{F_{1:h}} \E \left\|\sum_{k=1}^{h} F_k Y_{t-k}-\E[Y_t | Y_{1:t-1}] \right\|^2\\
     &
     =
     \sum_{t=1}^{T-1} \E \min_{F} \left\|(FE_h - M_\star)\mathbf{Z}^{T-t}Y_{1:T-1} \right\|^2\\
     &
     \geq
    \sum_{t=1}^{T-1} \E \min_{F} \left\|(FE_h - M_\star)\mathbf{Z}^{T-t}\mathbf{C} X_{1:T-1} \right\|^2\\
     &
     =
     \sum_{t=1}^{T-1} (\VEC M_\star)_1^\top [ \mathbf{R}_{11} -\mathbf{R}_{12} \mathbf{R}_{22}^{-1}\mathbf{R}_{21} ](t) (\VEC M_\star)_1
    \end{aligned}
\end{equation}
where $\mathbf{R}$ and $\mathbf{C}$ are as in \eqref{eq:rcdef}. Henceforth, we fix a single possible generative model \eqref{ass:ssmodel} and drop the dependency on $\scrP$ in $\ell_T(\sfM_h)=\ell_T(\sfM_h,\scrP)$ with $\scrP$ described by \eqref{ass:ssmodel}. We have established the following.

\begin{proposition}\label{prop:poldsrelax}
    Impose  \Cref{ass:ss}, and let 
    \begin{equation}\label{eq:rcdef}
        \begin{bmatrix}
              \mathbf{R}_{11} & \mathbf{R}_{12} \\ \mathbf{R}_{21} & \mathbf{R}_{22}
      \end{bmatrix}(t)=\mathbf{R}(t) = \mathbf{Z}^{T-t} \mathbf{C} \left(\E \left[  X_{1:T-1}X_{1:T-1}^\top \right]\right)\mathbf{C}^\top \mathbf{Z}^{T-t,\top} \quad \textnormal{with}\quad \mathbf{C}=\blkdiag(C_\star).
    \end{equation}
    For every class of linear filters $\mathsf{M}_h$ we have that:
    \begin{equation}
        \ell(\sfM_h) \geq \sum_{t=1}^{T-1} (\VEC M_\star)_1^\top [ \mathbf{R}_{11} -\mathbf{R}_{12} \mathbf{R}_{22}^{-1}\mathbf{R}_{21} ](t) (\VEC M_\star)_1.
    \end{equation}
\end{proposition}

The question now is whether the quadratic form $(\VEC M_\star)_1^\top [ \mathbf{R}_{11} -\mathbf{R}_{12} \mathbf{R}_{22}^{-1}\mathbf{R}_{21} ](t) (\VEC M_\star)_1$ is uniformly bounded in time or not. We shall see that there are simple examples in which it is not unless the history $h$ is allowed to grow sufficiently rapidly. Namely, let us consider noisy observations of the following scalar random walk model:
\begin{equation}\label{eq:polds2}
    \begin{aligned}
        X_{t+1}&=X_t +W_t, & \qquad
        Y_t & = X_t +V_{t+1}.
    \end{aligned}
\end{equation}

%The learner has access to the observations $Y_{1:T}$ and seeks to minimize the prediction risk
%\begin{equation}
%   \ell(f) \triangleq \frac{1}{T-1} \sum_{t=1}^{T-1} \E \|Y_t-f_t( Y_{1:t-1}) \|^2
%\end{equation}
%over some hypothesis class $\scrF \ni f_t$. In particular, we will study the setting of learning Markov parameters.
%\begin{equation}
%    \ell(f)-\min_f \ell(f) = \ell(f)-\min_{f \in \sfM_{h_\star}}\ell(f) + \min_{f \in \sfM_{h_\star}} \ell(f)-\min_f \ell(f)
%\end{equation}
%Since the class $\sfM_{h_\star}$ is linear, we further have:

%\begin{lemma}
%\begin{equation}
%     \min_{f \in \sfM_{h_\star}} \ell(f)-\min_f \ell(f) = \min_{F_{1:k}} \frac{1}{T-1} \sum_{t=1}^{T-1} \E \left\|\sum_{k=1}^{h} F_k Y_{t-k}-\E[Y_t | Y_{1:t-1}] \right\|^2
%\end{equation}
%\end{lemma}

Our next result shows that it is not just the number of unstable modes that matter in determining how many parameters are required, but also the memory length of the process $Y_{1:T}$. 

\begin{theorem}\label{thm:stablepoldsthm}
Impose \Cref{ass:ss} and suppose that $A_\star = C_\star =1$ as in \eqref{eq:polds2}. Let $\rho=A_\star -L_\star C_\star =1-L_\star$.  For every $ h = o\left( \frac{\log T}{\log \left(1-\rho\right)}\right)$ we have that
\begin{equation}
   \lim_{T\to \infty}  T^{-1}\ell(\sfM_h) = \infty.
\end{equation}
\end{theorem}

The result states that we require a context length or history at least of order $\frac{\log T}{1-\rho}$ for a length $T$ task with with $\rho \in (0,1)$. It is interesting to note that when the variance of the $V_t$ grows large, it can be analytically verified that $\rho$ tends to $1$. This offers the following interpretation: a poor signal to noise ratio in the filtering task corresponding to the generative model appearing in \Cref{ass:ssmodel} leads to a large required parameter dimension (context length).

\begin{proof}
Via \eqref{prop:poldsrelax} and \Cref{lem:marginalquadlemma} we have that:
\begin{equation}\label{eq:stablepoldsthmcalc1}
    \begin{aligned}
        &\min_{F_{1:h}}\frac{1}{T}\sum_{t=1}^{T-1} \E \left\|\sum_{k=1}^{h} F_k Y_{t-k}-\E[Y_t | Y_{1:t-1}] \right\|^2\\
     &
    \geq 
    \frac{1}{T}\sum_{t=1}^{T-1} 
         (\VEC M_\star)_1^\top [ \mathbf{R}_{11} -\mathbf{R}_{12} \mathbf{R}_{22}^{-1}\mathbf{R}_{21} ](t) (\VEC M_\star)_1 && (\textnormal{\Cref{prop:poldsrelax}}) \\
         &
         \gtrsim
         \frac{1}{T}\sum_{t=1}^{T-1} 
         \sum_{l=1}^{t-k}\left( \sum_{j=1}^{t-k-l}  \rho^{j-1}\right)^2 - \frac{1}{h+1} \left(\sum_{j=1}^{t-h} j \rho^{t-k-j} \right)^2 &&(\textnormal{\Cref{lem:marginalquadlemma}})\\
         & \gtrsim
          T^{-1} \left(\frac{T^2}{(1-\rho)^2}-O(1)\right)\\
         &\asymp \frac{T}{(1-\rho)^{2h}}.
    \end{aligned}
\end{equation}
Note that the RHS of \eqref{eq:stablepoldsthmcalc1} diverges for $h = o\left( \frac{\log T}{\log (1-\rho)}\right)$.
\end{proof}

\section{Discussion}

We have proposed a mechanistic explanation of emergence in a relatively simple class of autoregressive learning models. Crucially, and somewhat in parallel to empirical observation \citep{wei2022emergent}, we find that tasks requiring long-range prediction (put differently: multi-step reasoning) are precisely those which "emerge" at a critical model scale.  We also note that our findings are not at all in contrast with the recent theoretical model offered by \cite{arora2023theory}. They take scaling laws for loss functions as a given \citep{kaplan2020scaling}, and illustrate how such scaling laws  can naturally lead to the emergence of more complex reasoning. In the present work we argue directly about the loss. Consequently, we offer a complementary perspective to theirs and try rather to understand whether certain tasks intrinsically require a critical scale.

%One criticism against studying emergence from the loss perspective is offered by \cite{schaeffer2024emergent}. They argue that the phenomenon can be explained in part or removed by choice of the evaluation metric or loss function. We note 1) that this question is far from settled  \citep{du2024understanding}, and 2) even if it were settled our contribution could still help explain why certain losses exhibit such behavior in the first place. Moreover, it is difficult to decouple the choice of loss and the data-generating distribution (this is made abundantly clear by the present example: the Gaussian KL scales linearly in the average covariance of the states).

Our work also begs a number of further interesting questions and future directions are abound. We believe that there are many opportunities in exploring LLM related phenomena through the lens of systems modelling. This has also been pointed out by e.g., \cite{soatto2023taming} and \cite{alonso2024state}. It would certainly be interesting to study more concrete emergent skills from this lens, such as in-context learning. \cite{garg2022can} show that standard transformer models---such as the GPT-2 family \citep{radford2019language}---can perform linear regression from \iid\ examples without explicit supervision. How does the situation change when the examples are drawn sequentially and possibly lack ergodicity? Another interesting phenomenon in which one may want to understand the role of ergodicity, and in which sequence modelling may help, are language model "hallucinations". \cite{kalai2024calibrated} find that there is no necessary statistical reason for these to occur in an \iid\ generative model---does this change if we adopt a structured sequential perspective?

Our study also has a number of interesting extensions to other model classes. It may for instance be worthwhile to instantiate the Markovianesque model of \cite{ildiz2024self} and see if similar results can be derived.  It may also be interesting to consider other function classes allowing for some degree of nonlinearity.  \cite{goel2024can} prove than an attention-style architecture can approximate a stabilizing Kalman filter with sufficient context length---can we find corresponding lower bounds? Arguably, one would also like to incorporate some degree of representation learning into the present analysis.  \cite{ildiz2024understanding} study how multiple tasks compete for "representation capacity" via the spectral properties of certain tasks. It is natural to ask how phenomena such as lack of ergodicity and instability affect this competition.

\section*{Acknowledgements}

The authors acknowledge support from a Swedish Research Council international postdoc grant, NSF award CPS-2038873, NSF award SLES-2331880, AFOSR Award FA9550-24-1-0102, NSF CAREER award ECCS-2045834 and NSF award EnCORE-2217062.

\addcontentsline{toc}{section}{References}

\bibliographystyle{plainnat}

\bibliography{main.bib}

\appendix

\section{Auxilliary Lemmata}

\begin{lemma}
\label{lem:marginalquadlemma}
    Fix $\rho \in \R$ and let $\theta = \begin{bmatrix}
        \rho^{T-k-1} & \cdots & \rho & 1
    \end{bmatrix}^\top \in \R^{T-k} $. Then 
    \begin{equation}
         \theta^\top  \mathbf{R}_{11}\theta = \sum_{l=1}^{T-k}\left( \sum_{j=1}^{T-k-l}  \rho^{j-1}\right)^2
    \end{equation}
    and 
    \begin{equation}
        \theta^\top (  \mathbf{R}_{21}\mathbf{R}_{22}^{-1}\mathbf{R}_{12}
        )\theta = \frac{1}{h+1} \left(\sum_{j=1}^{T-k} j \rho^{T-k-j} \right)^2
    \end{equation}
\end{lemma}

\begin{proof}
    Notice that $\mathbf{R}_{11}=\mathbf{L}_{11}\mathbf{L}_{11}^\top$ so that $\theta^\top  \mathbf{R}_{11}\theta = \| \mathbf{L}_{11}^\top \theta\|^2$. Direct calculation yields that
    \begin{equation}
    \mathbf{L}_{11}^\top \theta 
    =
        \begin{bmatrix}
    1 & 1 & \cdots & 1 \\
    0 & 1 & \ddots & 1 \\
    \vdots & 0 & \ddots & \vdots \\
    0 & \cdots & 0 & 1 
\end{bmatrix} 
\begin{bmatrix}
        \rho^{T-k-1} \\ \cdots \\ a \\ 1
    \end{bmatrix}
    =
    \begin{bmatrix}
        \sum_{j=1}^{T-k}  \rho^{j-1}\\
         \sum_{j=1}^{T-k-1}  \rho^{j-1}\\
         \vdots\\
         1+\rho\\
         1
    \end{bmatrix}
    \end{equation}
    In particular using the closed form expressions in  \Cref{lem:marginaltoeplemma} we have that
    \begin{equation}
       \theta^\top  \mathbf{R}_{11}\theta = \sum_{l=1}^{T-k}\left( \sum_{j=1}^{T-k-l}  \rho^{j-1}\right)^2
    \end{equation}
    as was required. The second expression follows similarly by \Cref{lem:marginaltoeplemma}.
\end{proof}

\begin{lemma}
\label{lem:marginaltoeplemma}
Consider the matrices     
\begin{equation}
\begin{bmatrix}
   \mathbf{L}_{11} &0 \\ \mathbf{L}_{21} & \mathbf{L}_{22} 
\end{bmatrix}
=
\mathbf{L} 
\triangleq
\begin{bmatrix}
    1 & 0 & \cdots &\cdots  & 0 \\
    1 & 1 & 0 &\cdots & 0\\
    \vdots & \ddots & \ddots & \cdots & 0\\
    1 & 1 & 1 & \cdots & 1 
\end{bmatrix} \quad \textnormal{and}\quad 
\begin{bmatrix}
   \mathbf{R}_{11} &\mathbf{R}_{12}  \\ \mathbf{R}_{21} & \mathbf{R}_{22} 
\end{bmatrix}
=
\mathbf{R} \triangleq \mathbf{L} \mathbf{L}^\top \\
=
\begin{bmatrix}
   \mathbf{L}_{11}\mathbf{L}_{11}^\top & \mathbf{L}_{11}\mathbf{L}_{21}^\top \\ \mathbf{L}_{21}\mathbf{L}_{11}^\top  & \mathbf{L}_{21}\mathbf{L}_{21}^\top + \mathbf{L}_{22} \mathbf{L}_{22}^\top
\end{bmatrix}.
\end{equation}
We have that
\begin{equation}
    \begin{aligned}
        \mathbf{R}_{22}^{-1} & =     \begin{bmatrix}
        1 - \frac{h}{h+1} & -1 & 0 & 0 & \cdots &\cdots & 0 \\
        -1 & 2 & -1 & 0 & \cdots &\cdots & 0 \\
        0 & -1 & 2 & -1 &0& \cdots & 0 \\
        0 & 0 & -1 & 2 & \ddots&\cdots & 0 \\
        \vdots & \vdots & \vdots & \ddots & \ddots & -1&0 \\
        0 & 0 & 0 & \cdots & -1& 2 & -1 \\
        0 & 0 & 0 & 0 & \cdots & -1 & 1
    \end{bmatrix} \\
        \mathbf{R}_{21}\mathbf{R}_{22}^{-1}\mathbf{R}_{12} 
        &=
        \frac{1}{1+h} 
        \begin{bmatrix}
        1 & 2 & 3 & \cdots & T-h \\
        2 & 4 & 6 &\cdots & 2(T-h) \\
        \vdots &\vdots &\cdots \cdots &\vdots \\
        T-h & 2(T-h) & \cdots & \cdots & (T-h)^2
    \end{bmatrix}
    \end{aligned}
\end{equation}

\end{lemma}

\begin{proof}
It is easy to see that
\begin{equation}
\mathbf{L}_{21} \mathbf{L}_{21}^\top = h \mathbf{1}\mathbf{1}^\top
\qquad
\textnormal{and}
\qquad
   ( \mathbf{L}_{22} \mathbf{L}_{22}^\top)^{-1} = \begin{bmatrix}
        1 & -1 & 0 & 0 & \cdots &\cdots & 0 \\
        -1 & 2 & -1 & 0 & \cdots &\cdots & 0 \\
        0 & -1 & 2 & -1 &0& \cdots & 0 \\
        0 & 0 & -1 & 2 & \ddots&\cdots & 0 \\
        \vdots & \vdots & \vdots & \ddots & \ddots & -1&0 \\
        0 & 0 & 0 & \cdots & -1& 2 & -1 \\
        0 & 0 & 0 & 0 & \cdots & -1 & 1
    \end{bmatrix}
\end{equation}
Hence the Sherman-Morrison rank-1-update-formula yields that
\begin{equation}
\begin{aligned}
    \mathbf{R}_{22}^{-1} &= ( \mathbf{L}_{22} \mathbf{L}_{22}^\top)^{-1}  + \frac{h(\mathbf{L}_{22} \mathbf{L}_{22}^\top)^{-1}\mathbf{1}\mathbf{1} ^\top(\mathbf{L}_{22} \mathbf{L}_{22}^\top)^{-1} }{1+h\mathbf{1}^\top ( \mathbf{L}_{22} \mathbf{L}_{22}^\top)^{-1}\mathbf{1} }
    \\
    &
    =
    \begin{bmatrix}
        1 & -1 & 0 & 0 & \cdots &\cdots & 0 \\
        -1 & 2 & -1 & 0 & \cdots &\cdots & 0 \\
        0 & -1 & 2 & -1 &0& \cdots & 0 \\
        0 & 0 & -1 & 2 & \ddots&\cdots & 0 \\
        \vdots & \vdots & \vdots & \ddots & \ddots & -1&0 \\
        0 & 0 & 0 & \cdots & -1& 2 & -1 \\
        0 & 0 & 0 & 0 & \cdots & -1 & 1
    \end{bmatrix}
    - \frac{h}{1+h} \begin{bmatrix}
        1 & 0 &0 &\cdots &0\\
        0 &0 &0 &\cdots & 0\\
        \vdots & \ddots & \cdots &\ddots &\vdots\\
         0 &0 &0 &\cdots & 0
    \end{bmatrix}
\end{aligned}
\end{equation}
this yields the desired expression for $ \mathbf{R}_{22}^{-1}$.

Next, we have that
\begin{equation}
    \mathbf{R}_{21}(\mathbf{L}_{22} \mathbf{L}_{22}^\top)^{-1}\mathbf{R}_{12} = \begin{bmatrix}
        1 & 2 & 3 & \cdots & T-h \\
        2 & 4 & 6 &\cdots & 2(T-h) \\
        \vdots &\vdots &\cdots \cdots &\vdots \\
        T-h & 2(T-h) & \cdots & \cdots & (T-h)^2
    \end{bmatrix}
\end{equation}
and 
\begin{equation}
    \mathbf{R}_{21} \left(\frac{h}{1+h} \begin{bmatrix}
        1 & 0 &0 &\cdots &0\\
        0 &0 &0 &\cdots & 0\\
        \vdots & \ddots & \cdots &\ddots &\vdots\\
         0 &0 &0 &\cdots & 0
    \end{bmatrix}\right)
        \mathbf{R}_{12}
        = \frac{h}{1+h} 
        \begin{bmatrix}
        1 & 2 & 3 & \cdots & T-h \\
        2 & 4 & 6 &\cdots & 2(T-h) \\
        \vdots &\vdots &\cdots \cdots &\vdots \\
        T-h & 2(T-h) & \cdots & \cdots & (T-h)^2
    \end{bmatrix}
\end{equation}
Consequently
\begin{equation}
    \mathbf{R}_{21}\mathbf{R}_{22}^{-1}\mathbf{R}_{12} = \frac{1}{1+h} 
        \begin{bmatrix}
        1 & 2 & 3 & \cdots & T-h \\
        2 & 4 & 6 &\cdots & 2(T-h) \\
        \vdots &\vdots &\cdots \cdots &\vdots \\
        T-h & 2(T-h) & \cdots & \cdots & (T-h)^2
    \end{bmatrix}
\end{equation}
as per requirement.  
\end{proof}

\end{document}